\documentclass[conference]{IEEEtran}   
\usepackage[utf8]{inputenc}            
\usepackage{amsmath,amssymb,amsthm}           
\usepackage{graphicx}                  
\usepackage{caption}
\usepackage{subcaption}
\usepackage{cite}
\usepackage{hyperref}
\usepackage{makecell} 
\usepackage{algorithm}
\usepackage{algpseudocode}
\usepackage{svg}
\algrenewcommand\algorithmicrequire{\textbf{Input:}}
\algrenewcommand\algorithmicensure{\textbf{Output:}}
\algrenewcommand\algorithmiccomment[1]{\textcolor{blue}{//~#1}}
\usepackage{booktabs}
\usepackage{xcolor} 
\usepackage{todonotes}
\IEEEoverridecommandlockouts
\newtheorem{theorem}{Theorem}
\newtheorem{problem}{Problem}
\newtheorem{definition}{Definition}
\newtheorem{proposition}{Proposition}
\usepackage{siunitx}
\usepackage{caption}
\usepackage[top=19.1mm,bottom=19.1mm,left=19.1mm,right=19.1mm]{geometry}
\usepackage[most]{tcolorbox} 

\newcommand{\tnfl}{\Gamma}

\newcommand{\B}{\mathcal{B}}

\title{\LARGE \bf 
Formal Safety Verification and Refinement for \\ Generative Motion Planners via Certified Local Stabilization
}

\author{Devesh Nath$^*$ \and Haoran Yin$^*$  \and Glen Chou
\thanks{The authors are with the Georgia Institute of Technology, Atlanta, GA, USA. \texttt{\{dnath7, hyin95, chou\}@gatech.edu}}
}

\begin{document}

\maketitle

\begin{abstract}
\looseness-1We present a method for formal safety verification of learning-based generative motion planners. Generative motion planners (GMPs) offer advantages over traditional planners, but verifying the safety and dynamic feasibility of their outputs is difficult since neural network verification (NNV) tools scale only to a few hundred neurons, while GMPs often contain millions. To preserve GMP expressiveness while enabling verification, our key insight is to imitate the GMP by stabilizing references sampled from the GMP with a small neural tracking controller and then applying NNV to the closed-loop dynamics. This yields reachable sets that rigorously certify closed-loop safety, while the controller enforces dynamic feasibility. Building on this, we construct a library of verified GMP references and deploy them online in a way that imitates the original GMP distribution whenever it is safe to do so, improving safety without retraining. We evaluate across diverse planners, including diffusion, flow matching, and vision-language models, improving safety in simulation (on ground robots and quadcopters) and on hardware (differential-drive robot).

\end{abstract}

\section{Introduction}

Motion planning has been transformed by generative models like diffusion and conditional flow matching (CFM) \cite{Xiao2023SafeDiffuser, Chisari2024PCCFM}, which learn multimodal trajectory distributions and enable generative motion planners (GMPs) that produce diverse plans from inputs like language or images \cite{Dai2025SafeFlow, Carvalho2024MPD, Nawaz2024CLFNODE, Bouvier2025DDAT}. 
However, ensuring that GMP-generated trajectories satisfy safety and dynamic feasibility is difficult: GMPs often contain millions of parameters, making neural network verification (NNV) \cite{Zhang2018CROWN} intractable, limiting their use in safety-critical settings \cite{Bouvier2025DDAT}. NNV provides hard guarantees via set-based reachability but only scales to controllers with a few hundred neurons \cite{Rober2024CARV, Everett2021NFL, Everett2020Partitioning, Jafarpour2024AkashNFL}. More scalable statistical methods \cite{DixitLWCPB23_acp, Dawson2022NeuralCBF, Dai2025SafeFlow, DBLP:conf/l4dc/LewJBP22} yield weaker probabilistic guarantees or require prohibitive samples over long horizons. Thus, existing work trades off between expressive large models lacking hard guarantees and small verifiable models unable to capture complex behaviors.

\looseness-1To bridge this gap, we propose \textbf{SaGe-MP} (\underline{Sa}fe \underline{Ge}nerative \underline{M}otion \underline{P}lanning), a method that provides hard safety and dynamic feasibility guarantees for GMP-generated motion plans. Our key insight is that while NNV tools cannot directly certify the GMP, they can certify a \textit{small} neural \textit{tracking controller} that locally stabilizes the system around GMP-sampled references.
Reachability analysis of the resulting closed-loop system yields hard assurances of safety and dynamic feasibility over a continuum of inputs. Here, the GMP acts only as an open-loop plan generator, while verification is performed on the closed-loop dynamics induced by tracking a \textit{fixed} GMP plan, resulting in a smaller computational graph that makes NNV tractable. By tracking GMP references under the true dynamics, the controller also projects potentially dynamically-infeasible GMP plans onto feasible trajectories. To preserve the original GMP behavior if possible, we develop a trajectory-library approach: multiple GMP references are sampled offline, certified as safely trackable via NNV, and deployed online in a way that mimics the potentially multimodal GMP output. In this sense, our method is a lightweight GMP refinement that enhances safety and dynamic feasibility without costly GMP retraining. Our contributions are:
\begin{figure}
    \centering
    \includegraphics[width=\linewidth]{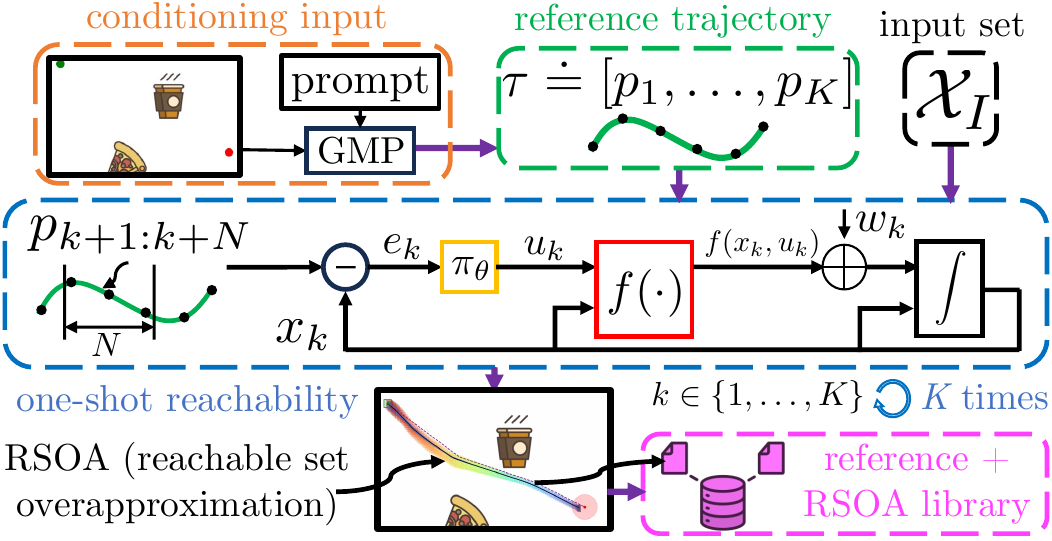}
    \vspace{-1.5em}
    \caption{An overview of our method, which samples GMP references and stabilizes them with a small neural tracking controller (Sec. \ref{sec:Controller}, verifies safety via NNV-based reachability (Sec. \ref{sec:T-NFL}-\ref{sec:RSOA}), and imitates the original GMP with a trajectory library (Sec. \ref{sec:efficient_fix}).}
    \label{fig:block}
    \vspace{-3pt}
\end{figure}

\begin{enumerate}
    \item An NNV-based method for formal safety verification of large GMPs that decouples trajectory generation from neural feedback loop verification, preserving planner expressiveness while providing hard safety guarantees.
    \item A method to certify a neural trajectory tracking controller that stabilizes dynamically-infeasible GMP references, producing safe, feasible trajectories.
    \item \looseness-1A trajectory-library method that stores certified-safe GMP plans and executes them online in a way that preserves the original behavior whenever safe. We prove sample complexity bounds for a target imitation error.
    \item Extensive simulation and real-world validation demonstrating safe stabilization of references from diverse generative models (diffusion, flow matching, VLMs, neural ODEs) on challenging nonlinear systems (e.g., 12D quadcopter, and learned NN dynamics).
\end{enumerate}

\vspace{-4pt}
\section{Related Work}
\vspace{-3pt}

\subsection{Generative Motion Planners (GMPs)}

GMPs generate diverse, human-like trajectories from rich inputs such as images or language \cite{Chisari2024PCCFM} and can be trained via imitation learning, unlike traditional trajectory optimization or sampling-based planning \cite{DBLP:books/cu/L2006}, which are limited to low-dimensional inputs. They also capture multimodal trajectory distributions, representing distinct homotopy classes that serve as a diverse set of solutions to the task.

\looseness-1
However, ensuring safety of GMP-generated trajectories is difficult. Existing methods encourage safety with planning costs \cite{Carvalho2024MPD} or use control barrier functions (CBFs) as corrections \cite{Mizuta2024CoBL, Nawaz2024CLFNODE, Xiao2023SafeDiffuser}, but often assume simplified dynamics (e.g., single integrator). For nonlinear robot dynamics \cite{DBLP:conf/amcc/DaiP23}, it is difficult to find CBFs; neural CBFs can be learned for such systems, but only heuristically encourage safety \cite{Dawson2022NeuralCBF}. In contrast, we provide hard guarantees for nonlinear NN dynamics by overriding unsafe GMP plans with a certified tracking controller that stabilizes around a nearby safe GMP plan.
Enforcing dynamic feasibility of GMP plans is also difficult. Predicting action sequences \cite{pmlr-v283-li25a} ensures feasibility but yields discontinuous outputs that GMPs predict poorly \cite{DBLP:conf/iclr/AjayDGTJA23}. Some methods enforce feasibility on linear systems \cite{DBLP:journals/corr/abs-2504-00236}. Others project GMP references onto the manifold of feasible plans under nonlinear dynamics \cite{Bouvier2025DDAT} but do not guarantee feasibility and have high replanning cost under model error. Our method instead stabilizes GMP references with a tracking controller, projecting them onto the dynamically-feasible manifold and rejecting errors.

\vspace{-3pt}
\subsection{Verification of NN Controllers}
\looseness-1\noindent \textbf{Formal methods.\quad} Reachability analysis verifies that system trajectories remain in a safe set \cite{DBLP:journals/arcras/AlthoffFG21}. Set-based methods propagate input and disturbance sets to compute reachable set overapproximations (RSOAs). For NN controllers, abstraction tools like CROWN \cite{Zhang2018CROWN} provide affine bounds on outputs and enable RSOA computation \cite{Rober2024CARV, Chen2023One-Shot, Jafarpour2024AkashNFL}. As the NN size grows, however, the computed RSOAs become highly conservative, revealing a tradeoff between model expressiveness and verification tractability. Thus, NNV-based reachability is limited to small NN controllers (hundreds of neurons) \cite{Rober2024CARV}. Symbolic and one-shot reachability \cite{Chen2023One-Shot, Rober2024CARV} reduces conservativeness \cite{Everett2021NFL}, but remain limited by NN size. Direct verification of GMPs, with millions of neurons, is intractable. We bridge this gap by stabilizing GMP samples with a small, verified tracking controller, preserving GMP expressiveness while certifying safety over a continuum of initial states.

\noindent\textbf{Statistical methods.\quad} Statistical methods that bound RSOA soundness probabilities scale better but give weaker guarantees. Conformal prediction (CP) \cite{DixitLWCPB23_acp} ensures reachable set coverage under an exchangeability assumption between the training and test distributions \cite{DBLP:conf/cdc/MuthaliSDLRFT23}, which feedback-controlled, time-varying data often violates. Adaptive CP relaxes this but only offers weaker long-run average coverage assurances. Simulation-based methods \cite{DBLP:conf/l4dc/LewJBP22, 10.1007/978-3-030-99524-9_17, jiang2017usingneuralnetworkscompute} approximate reachable sets probabilistically, but sample complexity grows quickly with horizon length. In contrast, our method provides hard guarantees robust to distribution shifts, scaling to GMP references hundreds of timesteps long.

\vspace{-5pt}
\section{Preliminaries}

We consider discrete-time nonlinear dynamical systems
\begin{equation}
    x_{k+1} = f(x_k, u_k, w_k),
    \label{eq:open_loop_dynamics}
\end{equation}
\looseness-1where $k \in \{1, \dots, K\}$ is the discrete timestep. The terms $x_k \in \mathcal{X} \subseteq \mathbb{R}^n$, $u_k \in \mathcal{U} \subseteq \mathbb{R}^m$, and $w_k \in \mathcal{W} \subseteq \mathbb{R}^d$ represent the state, control input, and an external disturbance, respectively. We assume that the initial state $x_I$ lies within a known set $\mathcal{X}_I \subseteq \mathcal{X}$ and that the disturbance $w_k$ is drawn from a compact set $\mathcal{W} \doteq \{w \mid \Vert w\Vert_\infty \le \overline{w}\}$, where $\overline{w}$ is known \textit{a priori}, though we note that these bounds can also be estimated from data \cite{DBLP:journals/ral/KnuthCOB21, 10161001}. Denote $\textsf{Int}(\underline{a}, \overline{a}) = \{a \mid \underline{a} \le a \le \overline{a}\}$, as a hyper-rectangular interval, where $\underline{a}, \overline{a} \in \mathbb{R}^A$ and the inequalities $\le$ are interpreted element-wise, i.e., $\underline{a}_i \le a_i \le \overline{a}_i$, $1 \le i \le A$. Denote the $a$-ball as $\B_a(c) \doteq \{ x \mid \Vert x - c \Vert_\infty \le a\}$.

\subsection{Computational Graph Robustness Verification}
A computational graph (CG) is a directed acyclic graph representing the sequence of mathematical operations on an input \cite{Rober2024CARV}, such as a neural network or the nonlinear function \eqref{eq:open_loop_dynamics}. For a graph $G$ with input set $\mathcal{Z} \subseteq \mathbb{R}^{n_i}$ and output $G(\mathcal{Z}) \subseteq \mathbb{R}^{n_o}$, we can compute a guaranteed overapproximation of $G(\mathcal{Z})$ using CROWN-based \cite{Zhang2018CROWN} tools like \texttt{auto\_LiRPA} \cite{Xu2020AutoLiRPA}. These tools provide guaranteed affine lower and upper bounds, $\underline{G}$ and $\overline{G}$, on the output $G(\mathcal{Z})$ for any hyper-rectangular input set $\mathcal{Z}$. This is formalized in the following proposition from \cite{Rober2024CARV}.

\vspace{-4pt}
\begin{proposition}[CG Robustness \cite{Xu2020AutoLiRPA}]
\label{thm:cg_robustness}
For CG $G$ and hyper-rectangular $\mathcal{Z} \doteq \{z \in \mathbb{R}^{n_i} \mid \underline{z} \le z \le \overline{z}\}$, there are affine functions $\underline{G}$ and $\overline{G}$ such that for all $z \in \mathcal{Z}$, $\underline{G}(z) \leq G(z) \leq \overline{G}(z)$. 
The inequalities hold element-wise and $\underline{G}(z) = \Psi z + \alpha$,  $\overline{G}(z) = \Phi z + \beta$, with $\Psi, \Phi \in \mathbb{R}^{n_o \times n_i}$ and $\alpha, \beta \in \mathbb{R}^{n_o}$.
\end{proposition}

\subsection{Reachability Analysis}
Reachability analysis computes the set of all states that a dynamical system can reach from given initial conditions, under admissible controls and disturbances. For brevity, the policy is written as $u_k=\pi_\theta(x_k, \eta_k)$, where $\theta$ are trainable parameters and $\eta_t \in \mathcal{I}$ are external conditioning inputs, e.g., a desired reference trajectory. The closed-loop dynamics are
\begin{equation}\label{eq:closed_loop_dynamics}
    x_{k+1} = f(x_k, \pi_\theta(x_k, \eta_k), w_k) \doteq \tilde{f}(x_k),
\end{equation}
where the dependence on control and noise is implicit in \eqref{eq:closed_loop_dynamics}.                          

\noindent \textbf{Exact Reachability:} We define the exact reachable set of $\tilde{f}$ at time $k+1$, assuming the state at time $k$ lies in the hyper-rectangular set $\mathcal{X}_{k} \doteq \{x \mid \underline{x}_{k} \le x \le \overline{x}_{k}\}$ as:
\begin{gather}
    R_{k+1}(\mathcal{X}_{k} \mid \eta_{k}) = \{x \mid x = \tilde{f}(x_{k}), \ x_{k} \in \mathcal{X}_{k}\} \notag \\
    = \left\{ x \;\middle|\;
    \begin{aligned}
        &x = f(x_{k}, \pi_\theta(x_{k}, \eta_{k}), w_{k}), \\
        &x_{k} \in \mathcal{X}_{k}, \ w_{k} \in \mathcal{W}
    \end{aligned}
    \right\}
    \label{eq:exact_reachability}
\end{gather}
Since computing exact reachable sets for nonlinear NN-controlled systems is intractable \cite{DBLP:journals/arcras/AlthoffFG21}, we compute reachable set over-approximations (RSOAs) $\hat{R} \supseteq R$ using NNV tools \cite{Zhang2018CROWN} and applying Proposition \ref{thm:cg_robustness}. We denote the certified region of attraction (ROA) as the set of initial states $\mathcal{A} \subseteq \mathcal{X}$ for which $\hat R_k \subseteq \mathcal{S}$ for all $k \in \{1,\ldots,K\}$ and $\hat R_K \subseteq \mathcal{X}_G$.

\looseness-1\noindent\textbf{RSOA Computation:} The RSOA can be computed symbolically (i.e., over time horizon $N\ge 1$ at once). The one-step case $N=1$ is commonly used \cite{DBLP:journals/arcras/AlthoffFG21}; however, increasing $N$ yields tighter bounds than one-step methods by avoiding per-step overapproximation error \cite{Rober2024CARV, Chen2023One-Shot}. Specifically, let $\tilde{f^{N}}(\cdot) = \tilde{f} \circ \dotsb \circ \tilde{f}(\cdot)$ denote $N$ compositions of $\tilde{f}$ and let the $F^N$ denote its CG. The $N$-step RSOA from timestep $k$ is:
\begin{equation}
    \hat{R}_{k+N}^N(\mathcal{X}_{k}\mid \eta_{k}) = \textsf{Int}\Big( \min_{x \in \mathcal{X}_{k}} \underline{F}^N(x), \max_{x \in \mathcal{X}_{k}} \overline{F}^N(x)\Big)
    \label{eq:RSOA}
\end{equation}

\vspace{-5pt}
\looseness-1\noindent where $\underline{F}^N$ and $\overline{F}^N$ are affine bounding functions for $F^N$ obtained using \texttt{auto\_LiRPA}. If $k=1$ and $N=K$, we compute the full $K$-step reachable set from $\mathcal{X}_I$, known as one-shot reachability \cite{Chen2023One-Shot}. Partitioning $\mathcal{X}_I$ into smaller sub-intervals yields tighter bounds \cite{Everett2020Partitioning}. While partitioning and symbolic methods yield tighter RSOAs, they are more computationally demanding than one-step methods \cite{Rober2024CARV, Everett2020Partitioning}.

\begin{figure*}
  \setlength{\intextsep}{0pt}
  \setlength{\textfloatsep}{0pt}
  \centering
  \includegraphics[width=\linewidth, trim=0pt 0pt 10pt 0pt, clip]{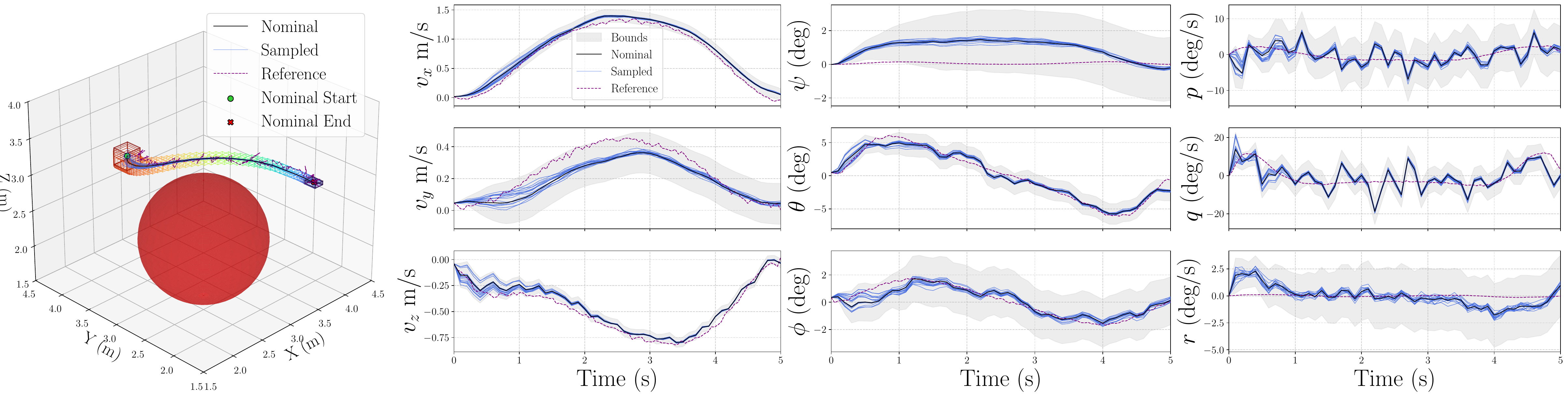}
  \vspace{-0.6em}
  \caption{RSOAs computed for a 3D quadcopter (with 12D dynamics), safely tracking a motion plan generated by a diffusion model.}
  \label{fig:3dquad_gt_plot}
\end{figure*}

\subsection{Generative Motion Planners} \label{sec:prelim_gmps}
\looseness-1We use the following motion planners $P: \mathcal{X}_I \times \mathcal{X}_G \times \mathcal{E} \rightarrow \mathcal{T}$ that map an initial state $x_I \in \mathcal{X}_I$, goal $x_g \in \mathcal{X}_G$, and task data $e \in \mathcal{E}$ (e.g., a natural language command) to generate a reference trajectory $\tau = \begin{bmatrix}p_1, \ldots, p_K\end{bmatrix} \in \mathcal{T} \subseteq \mathbb{R}^{n \times K}$. We denote learned parameters by $\nu$.

\noindent\textbf{Neural ODEs (NODEs)} \cite{DBLP:conf/nips/ChenRBD18} learn a neural vector field $\dot{x}_t = \phi_\nu(x_t, t)$ which is numerically integrated from $x_I$ to generate a discrete-time trajectory.

\looseness-1\noindent\textbf{Conditional Flow Matching (CFM)} \cite{song2022CFM} learns an ODE $\dot{z_t} = \psi_\nu(z_t, t)$ transporting a latent sample $z_0$ from a prior to a target trajectory distribution, yielding $\tau=z_1$ by integration.

\noindent\textbf{Diffusion Models} \cite{ho2020denoising} learn a network $\epsilon_\nu$ to reverse a multi-step noising process, iteratively denoising an initial sample $z_T \sim \mathcal{N}(0,I)$ in order to generate a clean trajectory $\tau$.

\noindent\textbf{Vision-Language Models (VLMs)} \cite{geminiteam2023gemini} act as zero-shot planners that generate position waypoints from an image and text, which are linearly interpolated to create trajectory $\tau$.

\section{Problem Statement}\label{sec:problem_statement}
\looseness-1We are given dynamics \eqref{eq:open_loop_dynamics}, disturbance set $\mathcal{W}$, initial state set $\mathcal{X}_I \subseteq \mathcal{X}$, goal set $\mathcal{X}_G \subseteq \mathcal{X}$, a safe set $\mathcal{S} \subseteq \mathcal{X}$ encoding safety and task constraints, and a planner $P: \mathcal{X}_I \times \mathcal{X}_G \times \mathcal{E} \rightarrow \mathcal{T}$, which maps an initial state $x_I \in \mathcal{X}_I$, a goal $x_\textrm{goal} \in \mathcal{X}_G$, and environment/task data (e.g., images, natural language) to a $K$-step reference trajectory $\tau \in \mathcal{T} \subseteq \mathbb{R}^{n \times K}$. References from $P$ may not be dynamically feasible (i.e., consistent with \eqref{eq:open_loop_dynamics}) or safe (i.e., contained in $\mathcal{S}$). We assume nothing further about $P$; it may be stochastic and can represent both learned and traditional planners.
We aim to (1) design a controller that replicates the behavior of $P$ whenever it is safe to do so and (2) formally verify that the closed-loop system safely reaches $\mathcal{X}_G$. Specifically, we aim to solve:
\begin{problem}[Design]\label{prob:design}
\looseness-1Train an NN control policy $\pi_\theta: \mathcal{X} \times \mathcal{I} \rightarrow \mathcal{U}$ that, given the current state $x$ and parameters $\eta$, generates closed-loop trajectories from $\mathcal{X}_I$ that (1) resemble the open-loop trajectories produced by $P$ while (2) ensuring safety and dynamic feasibility.
\end{problem}
\begin{problem}[Verification]\label{prob:verification}
    Verify that for all $x_I \in \mathcal{X}_I$, the system reaches $\mathcal{X}_G$ at timestep $K$ under \eqref{eq:closed_loop_dynamics} while staying in $\mathcal{S}$ for all $k \in \{1, \ldots, K\}$.
\end{problem}

\section{Methodology}\label{sec:method}
We overview our method, SaGe-MP (Fig. \ref{fig:block}), which learns a trajectory tracking controller (Sec. \ref{sec:Controller}), formulates the closed-loop tracking dynamics for a fixed GMP plan as a CG (Sec. \ref{sec:T-NFL}), and performs NN verification on this CG to compute the RSOA and ensure safety (Sec. \ref{sec:RSOA}). Finally, we sample distinct GMP references to construct a library of certified-safe RSOAs and use it to provide a runtime algorithm to preserve the original GMP behavior (Sec. \ref{sec:efficient_fix}). 

\subsection{Learned tracking controller design}
\label{sec:Controller}
Since GMPs typically generate reference trajectories that are not dynamically-feasible, the resulting error must be counteracted, e.g., with replanning or a trajectory-tracking controller. We opt for the latter; in this section, we describe our method for training and formally verifying a tracking controller. Note that we train a \textit{single} small tracking controller that is \textit{trajectory-conditioned}; we do not train a new tracking controller for different references and show that our controller generalizes well to different references in Sec. \ref{sec:results}.

We consider tracking controllers $\pi_\theta: \mathcal{X} \times \mathcal{T} \rightarrow \mathcal{U}$ that map a state $x$ and reference trajectory $\tau \in \mathcal{T}$ to a control input. For brevity, at timestep $k$, we also refer to this controller as $u_k = \pi_{\theta}(e_k)$, where $\pi_{\theta}:\mathbb{R}^{n \times N} \rightarrow \mathbb{R}^m$ has learnable parameters $\theta$ and $e_k$ collects reference errors over a future horizon of $N$ references $p_{k+1},\ldots,p_{k+N}$:
\begin{equation}\label{eq:error_ref}
    e_k \doteq [p_{k+1} - x_k,\ p_{k+2} - x_k,\ \ldots,\ p_{k+N} - x_k],
\end{equation}
Conditioning on future errors makes $\pi_\theta$ focus on reducing tracking error and anticipating changes, specializing to local error correction instead of global planning, which is handled by the expressive GMP. 
Thus $\pi_\theta$ can remain small ($\le 100$ neurons), which aids verification by keeping the unrolled CG compact and the resulting RSOA less conservative (cf. baselines in Figs. \ref{fig:model_sweep}, \ref{fig:NODE_baseline_failure}).
We represent $\pi_\theta$ as a small multi-layer perceptron (MLP) with leaky-ReLU activations and train it to produce dynamically-consistent corrective actions by minimizing the loss $\mathcal{L}(\theta) = \lambda_1 \mathcal{L}_{\text{track}} + \lambda_2 \mathcal{L}_{\text{state}} + \lambda_3 \mathcal{L}_{\text{ctrl}}$, 
which combines a mean squared tracking error $\mathcal{L}_\text{track}$ with smooth penalties for state $(\mathcal{L}_\text{state})$ and control $(\mathcal{L}_\text{ctrl})$ limit violations over a horizon $N$:
\begin{align}
    \mathcal{L}_{\text{track}} &= \frac{1}{N} \textstyle\sum_{j=k}^{k+N} \| x_{j+1} - p_{j+1} \|^2, \\
    \mathcal{L}_{\text{state}} &= \frac{1}{N} \textstyle\sum_{j=k}^{k+N} \left( \sigma(x_{j+1} - x_U) + \sigma(x_L - x_{j+1}) \right), \\
    \mathcal{L}_{\text{ctrl}} &= \frac{1}{N} \textstyle\sum_{j=k}^{k+N} \left( \sigma(u_j - u_U) + \sigma(u_L - u_j) \right),
\end{align}
where $x_k$ is the closed-loop state at timestep $k$ under controller $\pi_\theta$ and $\sigma(\cdot)$ is the softplus function, which penalizes violations of hyper-rectangular state and control bounds. To optimize $\mathcal{L}$, we compute gradients by rolling out $N$-step closed-loop trajectories under the dynamics \eqref{eq:open_loop_dynamics}, which we assume is differentiable.

\subsection{Tracking Neural Feedback Loop (T-NFL)}\label{sec:T-NFL}
We define the Tracking Neural Feedback Loop (T-NFL), denoted $\tnfl$, as a CG that models the evolution of the closed-loop tracking dynamics over the execution horizon $K$. This graph is constructed by unrolling the $K$-fold recursive application of the single-step closed-loop dynamics, $x_{k+1} = f(x_k, \pi_\theta(e_k), w_k)$,  where $\pi_\theta$ is the tracking controller from Sec. \ref{sec:Controller} and $e_k$ denotes the $N$-step lookahead reference tracking error.
Specifically, the T-NFL is a function $\tnfl$ that maps an initial state $x_I \in \mathcal{X}_I$, an \textit{a priori} fixed reference $\tau \in \mathcal{T}$ sampled from the GMP, and a disturbance sequence $\mathbf{w} = \{w_1, \dots, w_{K}\} \in \mathcal{W}^K\subseteq \mathbb{R}^{d \times K}$ to the resulting closed-loop trajectory $\tnfl(x_I,\tau,\mathbf{w}) = [\tnfl_1(x_I,\tau,\mathbf{w}), \ldots, \tnfl_K(x_I,\tau,\mathbf{w})] \doteq \xi = \{x_1, \dots, x_K\} \in \mathbb{R}^{n \times K}$, i.e.,
\begin{equation}
\begin{split}
    \tnfl_k(x_I, \tau, \mathbf{w}) \doteq & \ f(\ldots, 
        f(f(x_I,\tau_{1:N}, w_1), \tau_{2:N+1}, w_2), \\
        & \ldots, \tau_{k-1:k+N-2}, w_k),
\end{split}
\end{equation}
where we denote $\mathcal{W}^K = \mathcal{W} \times \ldots \times \mathcal{W}$ as the $K$-fold Cartesian product of $\mathcal{W}$. In particular, at every timestep $k$, the next $N$ references $p_{k+1}, \ldots, p_{k+N}$ are provided to the tracking controller. To guarantee that $N$ valid future references are always available, we extend the reference trajectory by padding it with the final point $p_K$ until its length reaches $K+N$. The process for generating the closed-loop trajectory is detailed in Alg. \ref{alg:comp_graph} and shown in Fig. \ref{fig:block}.

\begin{algorithm}[ht]
    \caption{T-NFL}
    \label{alg:comp_graph}
    \begin{algorithmic}[1]
            \Require 
                Initial state $x_I$; 
                Reference trajectory $\tau$;
                Disturbance sequence $\mathbf{w} = \{w_1, \dots, w_{K}\}$
            \Ensure Closed-loop trajectory rollout $\xi$.

            \State $x_0 \gets x_I$, \quad $\xi \gets [\;]$
            
            \State \textbf{for} $k = 1, \dots, K$:
            \State \quad $x_{k+1} \gets f(x_k, \pi_{\theta}(p_{k+1}-x_k,\ldots, p_{k+N}-x_k), w_k)$
            \State \quad $\xi \gets [\xi, x_{k+1}]$
            \State \Return $\xi$
    \end{algorithmic}
\end{algorithm}

\subsection{Safety Verification with Reachability Analysis}
\label{sec:RSOA}
To formally verify the safety of a generated plan, we compute an RSOA for the full-horizon closed-loop trajectory using one-shot reachability analysis \cite{Chen2023One-Shot}. We leverage the T-NFL's structure as a single computational graph, $\tnfl$, which represents the $K$-timestep closed-loop system rollout, from $k=0,\ldots,K-1$. This graph is passed to the \texttt{auto\_LiRPA} library \cite{Xu2020AutoLiRPA}, which propagates set-based inputs through the graph to bound the final output set.

\looseness-1Specifically, we encode the uncertain inputs using hyper-rectangular sets. The set of initial states is defined as $\mathcal{X}_I = \{x \mid \|x - x_I\|_\infty \leq \epsilon \}$ and the set of possible disturbances at each step is defined as $\mathcal{W} = \{w \mid \|w\|_\infty \le \overline{w}\}$. In contrast, the reference trajectory $\tau$ is fixed and passed as a deterministic tensor. \texttt{auto\_LiRPA} \cite{Xu2020AutoLiRPA} passes these inputs through $\tnfl$ to find vector-valued affine bounding functions, $\underline{\tnfl}$ and $\overline{\tnfl}$, for the $K$-step closed-loop trajectory $\xi$. 
This enables us to compute the RSOA at each timestep $k$ by considering specific components of these functions, $\underline{\tnfl}_k$ and $\overline{\tnfl}_k$, which bound the state of the closed-loop system at timestep $k$ for all $\mathbf{w} \in \mathcal{W}^k$ and for all initial conditions $x_I \in \mathcal{X}_I$. This hyper-rectangular set, $\hat{R}_k^K$, is found by determining the range of these affine bounding functions over the input sets. Since the functions are affine and the input sets are hyper-rectangles, this optimization is tractable and computed efficiently by \texttt{auto\_LiRPA}:
\begin{equation}\small
    \hspace{-5pt}\hat{R}_k^K(\mathcal{X}_I \mid \tau) \doteq \textsf{Int}\Bigg( \min_{\substack{x_I \in \mathcal{X}_I \\ \mathbf{w} \in \mathcal{W}^k}} \underline{\tnfl}_k(x_I, \tau, \mathbf{w}), \max_{\substack{x_I \in \mathcal{X}_I \\ \mathbf{w} \in \mathcal{W}^k}} \overline{\tnfl}_k(x_I, \tau, \mathbf{w}) \Bigg)\hspace{-3pt}
    \label{eq:RSOA_step_k}
\end{equation}
We can use these RSOAs $\hat{R}_k^K(\mathcal{X}_I \mid \tau)$ to formally verify reach-avoid properties for the closed-loop system. Specifically, we can ensure that the closed-loop system reaches goal set $\mathcal{X}_G$ by checking if $\hat{R}_K^K(\mathcal{X}_I \mid \tau) \subseteq \mathcal{X}_G$ and that the system is safe and satisfies task constraints by checking if $\hat{R}_k^K(\mathcal{X}_I \mid \tau) \subseteq \mathcal{S}$ for all $k \in \{1, \ldots, K\}$. These guarantees hold for all disturbance sequences that may be drawn from $\mathcal{W}^K$ and all initial conditions drawn from $\mathcal{X}_I$.

\subsection{RSOA libraries for multimodality and improved safety}
\label{sec:efficient_fix}

\begin{algorithm}[ht]
    \caption{Generation of RSOA library $\mathcal{C}$}
    \label{alg:safe_sampling}
    \begin{algorithmic}[1]
        \Require
            GMP $P$; initial set $\mathcal{X}_I$; goal set $\mathcal{X}_G$;
            RSOA budget $C$;
            environment $\mathcal{E}$; disturbance set $\mathcal{W}$;
        \Ensure A set of safe GMP references and RSOAs, $\mathcal{C}$.
        
        \State $\mathcal{C} \leftarrow \emptyset$
        \State \textbf{while} $|\mathcal{C}| < C$:
        \State \quad $\tau^c \leftarrow$ \Call{$P$}{$\mathcal{X}_I$, $\mathcal{X}_G$, $\mathcal{E}$}
        \State \quad \textbf{if not} \Call{Collision}{$\tau, \mathcal{E}$}:
        \State \qquad $\hat R^K(\mathcal{X}_I \mid \tau^c)$, success $\leftarrow$ \Call{FindRSOA}{$\mathcal{X}_I, \tau^c, \mathcal{W}$}
        \State \qquad \textbf{if} success \textbf{and not} \Call{Collision}{$\hat R^K(\mathcal{X}_I \mid \tau^c)$, $\mathcal{E}$}:
        \State \qquad \quad Add $(\tau^c, R^K(\mathcal{X}_I \mid \tau^c))$ to $\mathcal{C}$
        \State \textbf{return} $\mathcal{C}$
    \end{algorithmic}
\end{algorithm}

A key advantage of GMPs is their multimodality: producing diverse, viable trajectories for a task, potentially from different homotopy classes. GMPs are typically trained on multimodal expert demonstration data which encodes desirable behaviors. However, when imitated by the GMP, the resulting plans may not be safe or dynamically feasible. Our goal in Sec. \ref{sec:efficient_fix} is to build on the method of Sec. \ref{sec:RSOA} to balance two objectives: 1) preserve similarity to the GMP when safe while 2) increasing the safety rate
when starting from the certified ROA $\mathcal{A}$ with our method relative to the executing the raw GMP samples, avoiding costly retraining.

\noindent\textbf{Preserving multimodality.\quad}While following the closed-loop tracking policy defined by a single certified RSOA (and corresponding reference trajectory) ensures robust constraint satisfaction, it also collapses the closed-loop system into a unimodal trajectory distribution, which may not reflect the potential multimodality of the original GMP distribution.

To address this, we build upon Sec. \ref{sec:RSOA} by building a library $\mathcal{C}$ of $C$ reference trajectories sampled from the GMP and corresponding safe RSOAs $\mathcal{C} \doteq \{(\tau^i, \hat R_{1:K}^K(\mathcal{X}_I \mid \tau^i)\}_{i=1}^C$. At runtime, our method chooses a reference from $\mathcal{C}$ to track in order to better imitate the GMP distribution. 
Our method for building $\mathcal{C}$ is presented in Alg. \ref{alg:safe_sampling}. It repeatedly samples references $\tau^c$ from the GMP $P$ and, after an initial constraint check (ensuring that $\tau^c_K \in \mathcal{X}_G$ and $\tau^c_k \in \mathcal{S}$ for all $k$) and computes the closed-loop RSOA $\hat R_{1:K}^K(\mathcal{X}_I \mid \tau^c)$ for tracking $\tau^c$ under $\pi_\theta(x,\tau^c)$. If the RSOA satisfies the constraints, i.e., $\hat R_K^K(\mathcal{X}_I \mid \tau^c)\subseteq \mathcal{X}_G$ and $\hat R_k^K(\mathcal{X}_I \mid \tau^c)\subseteq \mathcal{S}$ for all $k$, it is added to the library $\mathcal{C}$; if not, the process repeats until $C$ safe RSOAs are found; $C$ is a maximum RSOA computation budget and can be tuned based on observed imitation error. For efficiency, this process can be parallelized on the GPU across multiple references $\tau^c$.

\looseness-1At execution time, a new reference $\tau^\textrm{samp}$ starting from $x_I^\textrm{samp}$ is sampled from the GMP $P$. We select the closest reference trajectory in $\mathcal{C}$, as measured by the minimum $\ell_2$ error
\begin{equation}\label{eq:closest_error}
E_\tau(\tau^\textrm{samp}, \mathcal{C}) \doteq \min_{i\in \{1,\ldots,C\}} \textstyle\sum_{k=1}^{K}\Vert \tau_k^i -\tau_k^\textrm{samp}\Vert_2^2,
\end{equation}
where we denote $i^*$ as the $\arg\min$ of \eqref{eq:closest_error}. Finally, to obtain a safely-stabilizable, dynamically-feasible trajectory, we evaluate the closed-loop dynamics obtained by tracking $\tau^{i^*}$ under $\pi_\theta$, yielding the closed-loop trajectory $\xi =\tnfl(x_I^\textrm{samp}, \tau^{i^*}, \mathbf{w})$. 

\noindent \textit{Theoretical Analysis.\quad} In the following, for the specific case where the open-loop dynamics \eqref{eq:open_loop_dynamics} are deterministic, i.e., $x_{t+1} = f(x_t, u_t)$, we derive finite-sample bounds on the rate at which the closed-loop imitation error 
\begin{equation}\label{eq:closed_loop_imitation_error}
    E_\xi(\tau^\textrm{samp}, \mathcal{C}) \doteq \Vert \tnfl(x_I^\textrm{samp}, \tau^{i^*}) - \tnfl(x_I^\textrm{samp}, \tau^\textrm{samp})\Vert
\end{equation} decreases as the size of the library $\mathcal{C}$ increases. At a high level, this result shows that if $\tau^\textrm{samp}$ can be safely stabilized under $\pi_\theta$, as $C$ increases, $\tnfl(x_I^\textrm{samp}, \tau^{i^*})$ converges to $\tnfl(x_I^\textrm{samp}, \tau^\textrm{samp})$ -- the trajectory that would have been executed by directly tracking $\tau^\textrm{samp}$. That is, as the library grows, our method becomes less invasive and more closely approximates the GMP trajectory distribution.
\begin{definition}[$\mu$-safe stabilization]
	A reference $\tau$ starting from $x_I$ to be $\mu$-safely stabilizable if $\hat R_k^K(\B_\mu(x_I) \mid \tau) \subseteq \mathcal{S}$ for all $k \in \{1, \ldots, K\}$ and $\hat R_K^K(\B_\mu(x_I) \mid \tau) \subseteq \mathcal{X}_G$.
\end{definition}

Here, $L$ is the Lipschitz constant of the T-NFL with respect to the reference, i.e., $L$ satisfies $\Vert \tnfl(x_I,\tau) - \tnfl(x_I,\tau^\textrm{samp})\Vert \le L\Vert \tau - \tau^\textrm{samp}\Vert$. The following result proves a bound on $C$ such that with probability at least $1-\alpha$, the error between the original and certified closed-loop trajectories \eqref{eq:closed_loop_imitation_error} $\Vert \xi^\textrm{samp} - \xi^{i^*}\Vert$ is at most $\epsilon$, i.e., $\textrm{Pr}\big[E_\xi(\tau^\textrm{samp}, \mathcal{C} ) \le \epsilon\big] \ge 1-\alpha$:
\begin{theorem}[Minimal invasiveness]
	Assume that A) the open-loop dynamics \eqref{eq:open_loop_dynamics} are deterministic, B) \eqref{eq:open_loop_dynamics} and $\pi_\theta$ are Lipschitz continuous in all arguments, and C) Alg. \ref{alg:safe_sampling} samples reference trajectories i.i.d. from the GMP $P$. Let $\epsilon$ be a pre-specified error threshold and $\tau^\textrm{samp}$ be an $\mu$-safely stabilizable reference generated from $P$ with the properties that A) for $\delta = \epsilon/L$, $\textrm{Pr}(P \textrm{ draws sample from } \B_\delta(\tau^\textrm{samp})) \doteq p_\textrm{samp} \in (0, 1]$ and that B) for all $\tau \in \B_\delta(\tau^\textrm{samp})$, $\tau$ is also $\mu$-safely stabilizable. After $C$ i.i.d. draws from $P$, 
	\begin{equation}
		\textrm{Pr}\big[E_\xi(\tau^\textrm{samp}, \mathcal{C}) \le \epsilon\big] \ge 1-\exp(-p_\textrm{samp}C),
	\end{equation}
	and to ensure that $\textrm{Pr}\big[E_\xi(\tau^\textrm{samp}, \mathcal{C}) \le \epsilon\big] \ge 1-\alpha$ for $\alpha \in (0, 1)$, one can select $C \ge \log(1/\alpha)/p_\textrm{samp}$.
\end{theorem}
\begin{proof}
	If $\tau^\textrm{samp}$ is $\mu$-safely stabilizable, and the probability of sampling $\tau \in \B_\delta(\tau^\textrm{samp})$ is $p_\textrm{samp}$ per draw, the chance of zero successes in $C$ i.i.d. trials is $1-(1-p_\textrm{samp})^C \ge 1-\exp(-p_\textrm{samp}C)$. When some $\tau \in \B_\delta(\tau^\textrm{samp})$ is sampled, it will be added to $\mathcal{C}$ as it is $\mu$-safely stabilizable. Then, at runtime, when $\tau^\textrm{samp}$ is sampled, from the definition of Lipschitz continuity, we have that $\Vert \tnfl(x_I,\tau) - \tnfl(x_I,\tau^\textrm{samp})\Vert \le L\Vert \tau - \tau^\textrm{samp}\Vert$ for all $\tau$. $L$ is guaranteed to exist and be finite since $f$ and $\pi_\theta$ are Lipschitz and $f$ is deterministic. Since $\delta = \epsilon/L$ and there exists some $\tau$ in $\mathcal{C}$ with probability at least $1-\exp(-p_\textrm{samp} C)$ that satisfies $\Vert \tau - \tau^\textrm{samp}\Vert \le \delta$, we have that $E_\xi(\tau^\textrm{samp}, \mathcal{C}) \le \Vert \tnfl(x_I,\tau) - \tnfl(x_I,\tau^\textrm{samp})\Vert \le L\Vert \tau - \tau^\textrm{samp}\Vert \le L\delta \le \epsilon$.
\end{proof}
In Sec. \ref{sec:results}, we empirically validate that by collecting even a small number of RSOAs in $\mathcal{C}$, we can reduce the trajectory distribution shift with respect to the GMP $P$.

\begin{algorithm}[ht]
    \caption{Safe Multimodal Runtime Execution}
    \label{alg:safe_runtime_execution}
    \begin{algorithmic}[1]
        \Require 
            initial state $x_I$; 
            GMP reference $\tau^\textrm{samp}$;
            library $\mathcal{C}$;
            tracking controller $\pi_\theta$;
            disturbances $\mathbf{w} = \{w_1, \dots, w_{K}\}$
        \Ensure Closed-loop trajectory rollout $\xi$.

        \State $E_\tau(\tau^\textrm{samp}, \mathcal{C}), i^* \gets$ Solve \eqref{eq:closest_error} 
        \State $\tau^{i} \gets$ Reference $i^*$ from $\mathcal{C}$
        
        \State $\xi \gets \tnfl(x_I, \tau^{i^*}, \mathbf{w})$
        
        \State \Return $\xi$
    \end{algorithmic}
\end{algorithm}

\noindent\textbf{Verifying large sets of initial conditions.\quad}
\looseness-1We can also extend our approach to verify safe forward invariance for large initial condition sets $\mathcal{X}_I$ by blending multiple GMP-generated reference trajectories with the tracking controller $\pi_\theta$. In contrast, na\"ively stabilizing a large $\mathcal{X}_I$ with $\pi_\theta$ around a single reference causes the RSOA to grow out of bounds in practice.

To prevent this, we sample multiple reference trajectories $\tau^i$ from the GMP, each with different initial conditions $x_I^i \in \mathcal{X}_I$, and stabilize a smaller neighborhood $\mathcal{X}_I^i \doteq \B_\epsilon(x_I^i) \subseteq \mathcal{X}_I$ of initial conditions to $\tau^i$ using $\pi_\theta$. This relies on only local trajectory robustness, enabling the tracking controller to remain relatively simple. As in Alg. \ref{alg:safe_sampling}, we can form a library of references and RSOAs $\mathcal{C} \doteq \{(\tau^i, \hat R^K(\mathcal{X}_I^i \mid \tau^i)\}_{i=1}^C$. Then, by selecting these initial sets to partition $\mathcal{X}_I$, i.e., $ \mathcal{X}_I = \bigcup_{i=1}^C \mathcal{X}_I^i$, at execution time, if $x_I \in \mathcal{X}_I^i$, we can use $\mathcal{C}$ and track $\tau^i$ using $\pi_\theta$; the corresponding RSOA certifies that $\tilde{f}$ stays in $\mathcal{S}$ for all $k \in \{1, \ldots, K\}$. 

\section{Results}\label{sec:results}
In this section, we compare to baselines (Sec. \ref{sec:results_baselines}), evaluating on nonlinear and learned unicycle and quadcopter dynamics (Sec. \ref{sec:results_unicycle}-\ref{sec:results_3dquad}), and demonstrating our method on multimodal GMPs and verification of large $\mathcal{X}_I$ sets (Sec. \ref{sec:results_multimodal}-\ref{sec:results_large}.
were conducted on a desktop computer equipped with an Intel i9-14900K CPU, 64GB of RAM, and an NVIDIA RTX 4090 GPU.
For all GMPs other than the pre-trained VLM, we generated training data via trajectory optimization. See App. \ref{app:parameters} for  further experiment details.

\subsection{Baseline Comparisons}\label{sec:results_baselines}
We motivate our method by highlighting three challenges in improving the safety of learned planners. First, large NN tracking controllers yield loose RSOAs as computed by NNV tools, motivating smaller controllers. Second, sampling-based verification of RSOAs is highly sample-inefficient over long horizons, motivating NNV tools. Finally, neural CBFs safety filters can be unsound and lead to violations, motivating formally verified controllers.

\noindent \textbf{Loose RSOAs for large NNs.\quad} We illustrate the tradeoff between NN size and closed-loop RSOA volume. As shown in Fig.\ref{fig:model_sweep}, larger NNs improve expert data recreation error (blue) and shrink the exact reachable set $R_k^K(\mathcal{X}_I)$ but enlarge the computed RSOA $\hat R_k^K(\mathcal{X}_I)$ due to the resulting NNV conservativeness (red). For example, Fig. \ref{fig:NODE_baseline_failure} shows a neural ODE planner with three hidden layers of width 256 yields a conservative RSOA that grows unbounded while the trajectories remain bounded.

\begin{figure}[htbp]
    \vspace{-8pt}
    \centering
    \begin{minipage}[c]{0.55\linewidth}
        \includegraphics[width=0.95\linewidth]{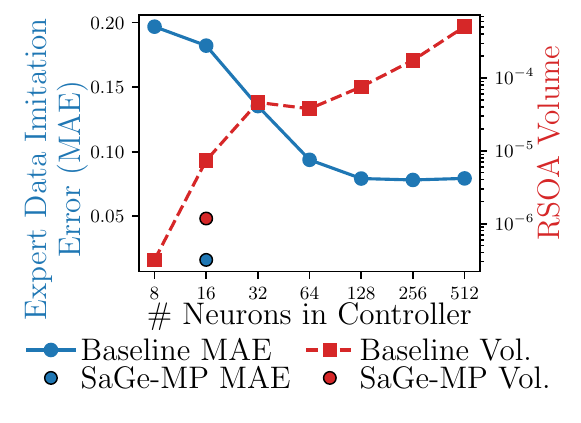}
    \end{minipage}
    \begin{minipage}[c]{0.45\linewidth}
        \caption{We visualize the trade-off between controller capacity and RSOA volume. While larger models (more neurons) reduce tracking error, they inflate the reachable set volume, leading to more conservative guarantees.}
        \label{fig:model_sweep} 
    \end{minipage}
    \vspace{-12pt}
\end{figure}

\begin{figure}[ht]
 \vspace{-5pt}
  \begin{minipage}[c]{0.63\linewidth}
    \includegraphics[width=0.97\linewidth]{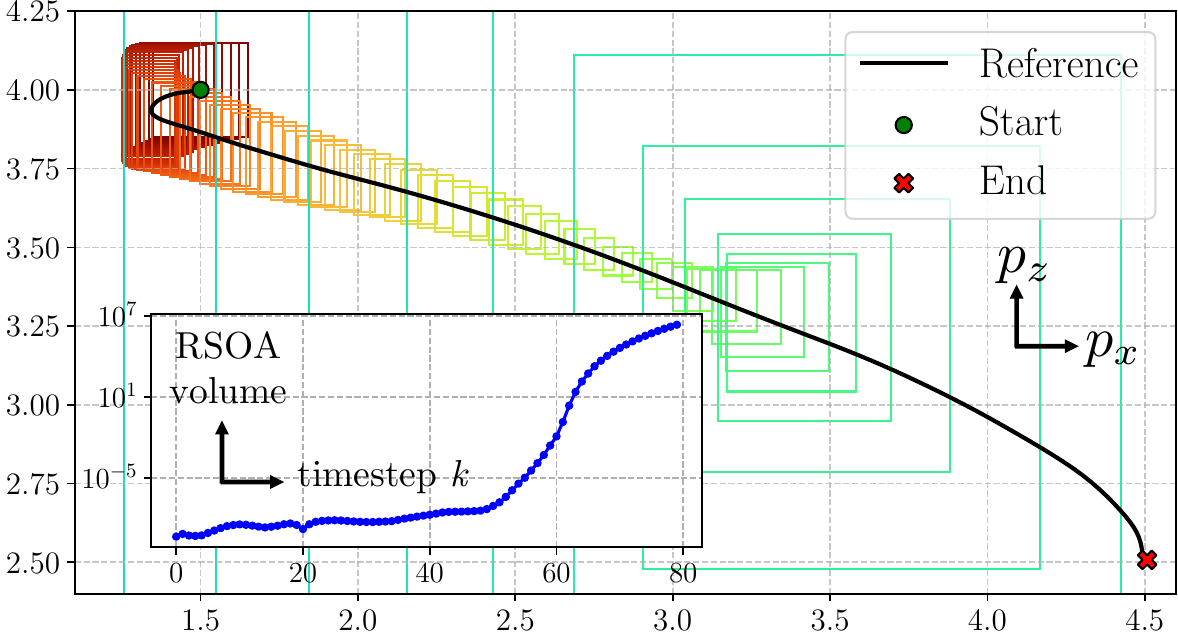}
  \end{minipage}\begin{minipage}[c]{0.37\linewidth}
    \caption{RSOA for neural ODE (756 neurons). The rapid RSOA growth shows a looseness issue for NNV with larger models; this is reflected in the RSOA volumes as well (inset).}
    \label{fig:NODE_baseline_failure}
  \end{minipage}
  \vspace{-3pt}
\end{figure}

\noindent \textbf{Inefficiency of sampling-based safety analysis.} Sampling-based methods are computationally inefficient at uncovering rare but critical failures. In a 2D quadcopter scenario (Fig. \ref{fig:2dquad_failure}), our NN tracking controller follows a $K = 100$-length reference trajectory (black) generated by a neural ODE while avoiding an obstacle (red). After $10^6$ simulations, only seven collisions were found, whereas the computed RSOA intersects the obstacle, immediately certifying the possibility of a violation. The high cost of sampling highlights its limitations and the advantage of reachability-based formal verification.

\begin{figure}[ht]
 \vspace{-3pt}
  \begin{minipage}[c]{0.55\linewidth}
    \includegraphics[width=0.95\linewidth]{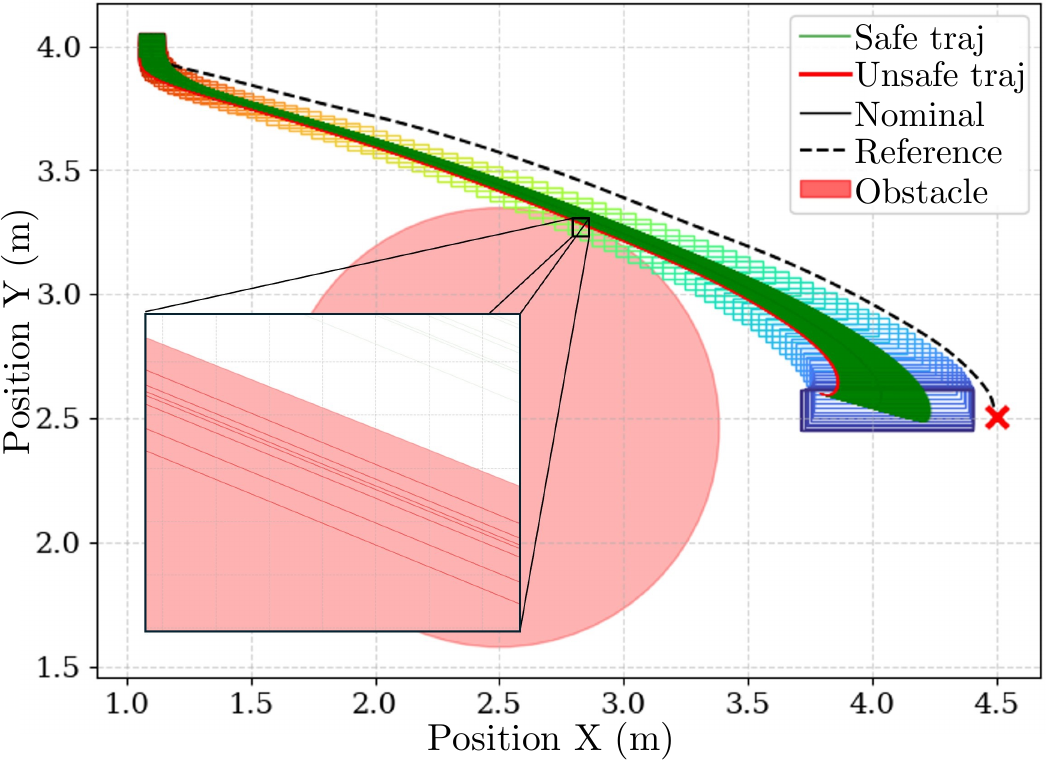}
  \end{minipage}\begin{minipage}[c]{0.45\linewidth}
    \caption{Limitations of sampling-based safety analysis. While $10^6$ random simulations uncovered just seven collisions over a $K=100$ horizon, our RSOA-based method can directly certify the potential for collision.}\label{fig:2dquad_failure}
  \end{minipage}
  \vspace{-8pt}
\end{figure}

\noindent \textbf{Unsafe behavior from learned safety filters.\quad} A baseline neural CBF \cite{Dawson2022NeuralCBF} fails to reliably enforce safety. In tracking a neural ODE-generated reference under unicycle dynamics (Fig. \ref{fig:cbf_fail}), the learned CBF overrides the controller when inputs appear unsafe, yet only 46.4\% of 1000 rollouts from $\mathcal{X}_I$ remain safe (orange). This is because typically neural CBFs only enforce the CBF conditions on sampled states \cite{Dawson2022NeuralCBF}, compromising the guarantees. In contrast, our method formally certifies that the closed-loop system is 100\% safe.

\begin{figure}[ht]
 \vspace{-3pt}
  \begin{minipage}[c]{0.63\linewidth}
    \includegraphics[width=0.97\linewidth]{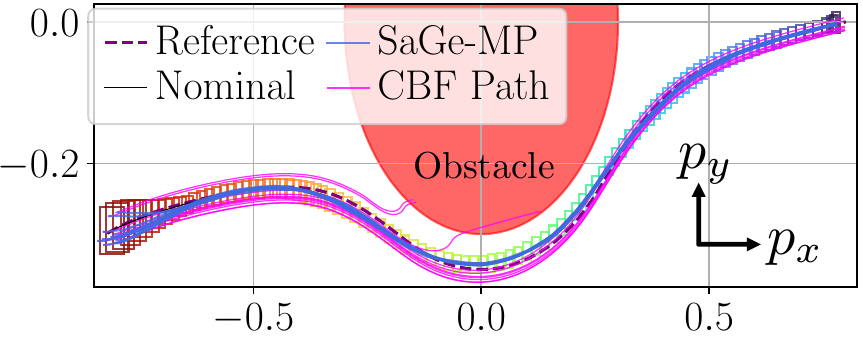}
  \end{minipage}\begin{minipage}[c]{0.37\linewidth}
    \caption{A neural CBF baseline (pink) frequently fails to ensure safety, while SaGe-MP (blue) guarantees safety.}
    \label{fig:cbf_fail}
  \end{minipage}
\vspace{-14pt}
\end{figure}

\subsection{Unicycle}\label{sec:results_unicycle}
\noindent\textbf{Simulation results on analytical and learned dynamics.\quad} We evaluate our method on unicycle dynamics \cite[Eq. 13.18]{DBLP:books/cu/L2006}, discretized with a $0.1s$ forward-Euler step, tracking trajectories for $K=100$ timesteps. To test applicability to learned models, we train an MLP (3 layers, 64 neurons) on $2\times10^6$ transitions from the ground truth. A NODE planner is trained on ground-truth trajectories, and a tracking controller $\pi_\theta$ with two hidden layers of width 8 is trained to follow them; the same controller is used in both cases.

\begin{figure}[h]
    \setlength{\intextsep}{0pt}
    \setlength{\textfloatsep}{0pt}
    \includegraphics[width=\linewidth]{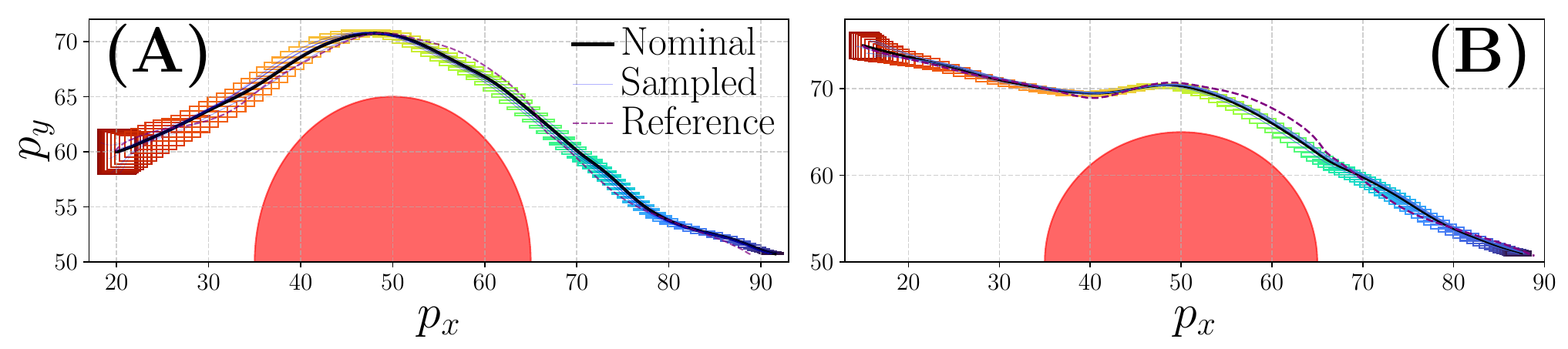}
    \vspace{-2em}
    \caption{Simulated unicycle RSOAs, with references generated by (A) NODE + analytical dynamics, (B) NODE + learned dynamics.}
    \label{fig:unicycle_combined}
    \vspace{-2pt}
\end{figure}

As shown in Fig. \ref{fig:unicycle_combined}, under both analytical (a) and learned (b) dynamics, our T-NFL safely tracks neural ODE reference trajectories, producing RSOAs that formally verify safety and goal-reaching. RSOA computation times were $100.78s$ and $51.95s$, respectively. This experiment shows that our method can follow NODE-generated plans and certify closed-loop safety for both analytical and learned dynamics.

\begin{figure}[h]
    \centering
    \includegraphics[width=0.9\linewidth]{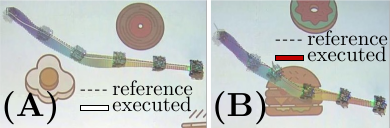}
    \caption{Hardware deployment: a plan is generated from an image of the scene and a prompt using a VLM, which is safely executed by our controller on an differential-drive robot.}
    \label{fig:real_robot_vlm}
    \vspace{-2pt}
\end{figure}

\noindent \textbf{Hardware results with natural language task constraints.\quad} We demonstrate our framework on a physical wheeled robot (Fig. \ref{fig:real_robot_vlm}). A VLM (Gemini 2.5 Pro) \cite{geminiteam2023gemini} generated reference trajectories from images and text prompts. We evaluate two examples: one where the VLM generates a collision-free path to the goal (Fig. \ref{fig:real_robot_vlm}a), and another where it reaches the target while moving over a food item and treating other objects as obstacles (Fig. \ref{fig:real_robot_vlm}b). See App. \ref{app:parameters} for prompt details. 

The VLM produces reference trajectories as discussed in Sec. \ref{sec:prelim_gmps}. The same tracking controller from the simulation experiments is used to follow these trajectories. At runtime on hardware, the vehicle stayed within the computed RSOA (Fig. \ref{fig:real_robot_vlm}). Despite being trained only on NODE-generated references, the controller successfully tracked VLM-generated trajectories, with formal verification providing guaranteed RSOA computation even out-of-distribution. These results show that our controller generalizes to unseen reference trajectory distributions and safely stabilizes plans from images and language, ensuring task completion from any state in $\mathcal{X}_I$.

\vspace{-6pt}
\subsection{Planar Quadcopter}\label{sec:results_2dquad}
Next, we evaluate our framework on a planar, 6D quadcopter model (App. \ref{app:parameters}), time-discretized with a timestep of $0.05$s for a horizon of $K = 100$. The task is to navigate a $5 \times 5 \; m^2$ workspace containing a circular obstacle.

\vspace{-8pt}
\begin{figure}[h]
  \setlength{\intextsep}{0pt}
  \setlength{\textfloatsep}{0pt}
  \centering
  \includegraphics[width=0.9\linewidth]{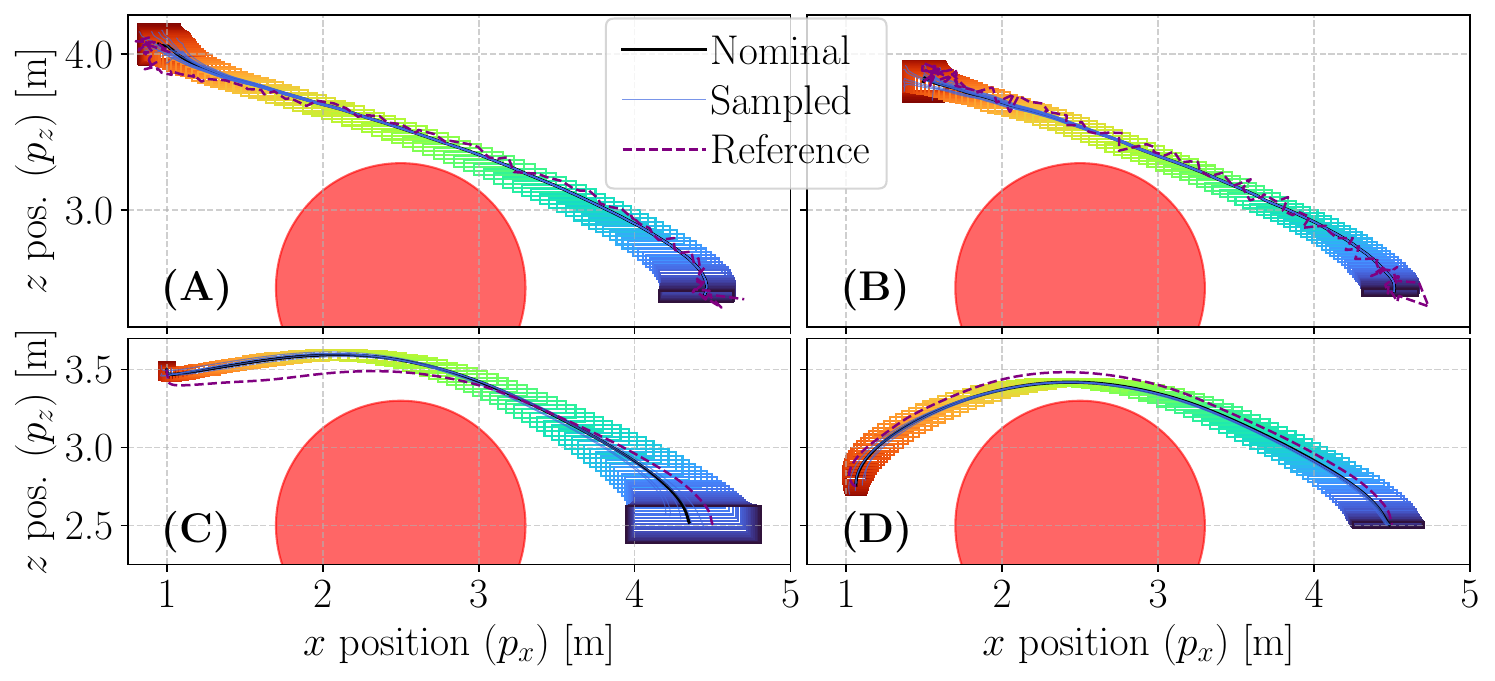}
  \vspace{-0.8em}
  \caption{\looseness-1RSOAs computed for the planar quadcopter for various planners, verifying safety: (A) diffusion + analytical dynamics (AD); (B) CFM + AD; (C) NODE + learned dynamics; (D) NODE + AD.}
  \label{fig:2dquad_combined}
\end{figure}

\noindent\textbf{Analytical dynamics.\quad}As shown in Fig. \ref{fig:2dquad_combined}, our framework certifies plans from diverse learned planners, including diffusion, CFM, and neural ODE models. The tracking controller (2 hidden layers, 8 neurons each) reliably follows the generated trajectories, yielding collision-free closed-loop RSOAs and certifying safety for all $x_I \in \mathcal{X}_I$. RSOA computation times are comparable to before, averaging $123.21$s for analytical dynamics and $59.85$s for the learned model, using 4 and 25 initial partitions, respectively. 

As a surrogate for measuring the increased safety rate when starting from the certified ROA with our method relative to the GMP, we consider a fixed input set $\mathcal{X}_I$ and compute empirical safety rates for 1) open-loop and 2) closed-loop execution of GMP plans. Out of 126 references sampled from the diffusion planner, 1) 78.5\% and 2) 94.4\% were safe starting from $\mathcal{X}_I$. For the same input set, we also evaluated the CFM-based planner, which yielded 68.5\% and 81.5\% safety out of 108 samples. In contrast, we are able to certify 100\% safety; i.e., $\mathcal{X}_I \subseteq \mathcal{A}$ for diffusion and CFM. 

Overall, this experiment suggests that SaGe-MP stabilizes noisy, dynamically-infeasible references from diffusion and CFM (Fig. \ref{fig:2dquad_combined}, purple), increasing the safety rate by certifying the ROA and effectively ``projecting" them onto the set of feasible trajectories that safely reach the goal, improving on the approximate projection of ~\cite{Bouvier2025DDAT}.

\noindent\textbf{Learned dynamics.\quad}To illustrate scalability to higher-dimensional learned dynamics, we trained an NN dynamics model (3 hidden layers, 64 neurons each) from transitions of the analytical model, along with a tracking controller (2 layers, 16 neurons each). The resulting closed-loop RSOAs again show safe tracking (Fig. \ref{fig:2dquad_combined}c). These results highlight the planner-agnostic nature of our method and ability to safely track trajectories from GMPs on underactuated dynamics.

\vspace{-1pt}
\subsection{3-D Quadcopter}\label{sec:results_3dquad}

We test scalability on an analytical 12-state 3D quadcopter model \cite{sabatino2015quadrotor}, discretized with a $0.05$s timestep over a $K=100$ horizon. The task is set in a $5 \times 5 \times 5$ m$^3$ workspace with a spherical obstacle of radius $1$m. We train a tracking controller $\pi_\theta$, with 2 hidden layers of 8 neurons each, to track a diffusion-generated reference. As shown in Fig. \ref{fig:3dquad_gt_plot}, the computed RSOA guarantees safety: the controller tracks the reference trajectory while the certified RSOA remains collision-free, computed in $2153.45$s with 4 uniform partitions of $\mathcal{X}_I$.
This demonstrates that our method can scale to certifiably stabilize GMP-generated references on high-dimensional, underactuated nonlinear systems.

\subsection{Safe, Multimodal Planning}\label{sec:results_multimodal}

\begin{figure}[h]
    \centering
    \includegraphics[width=\linewidth
    ]{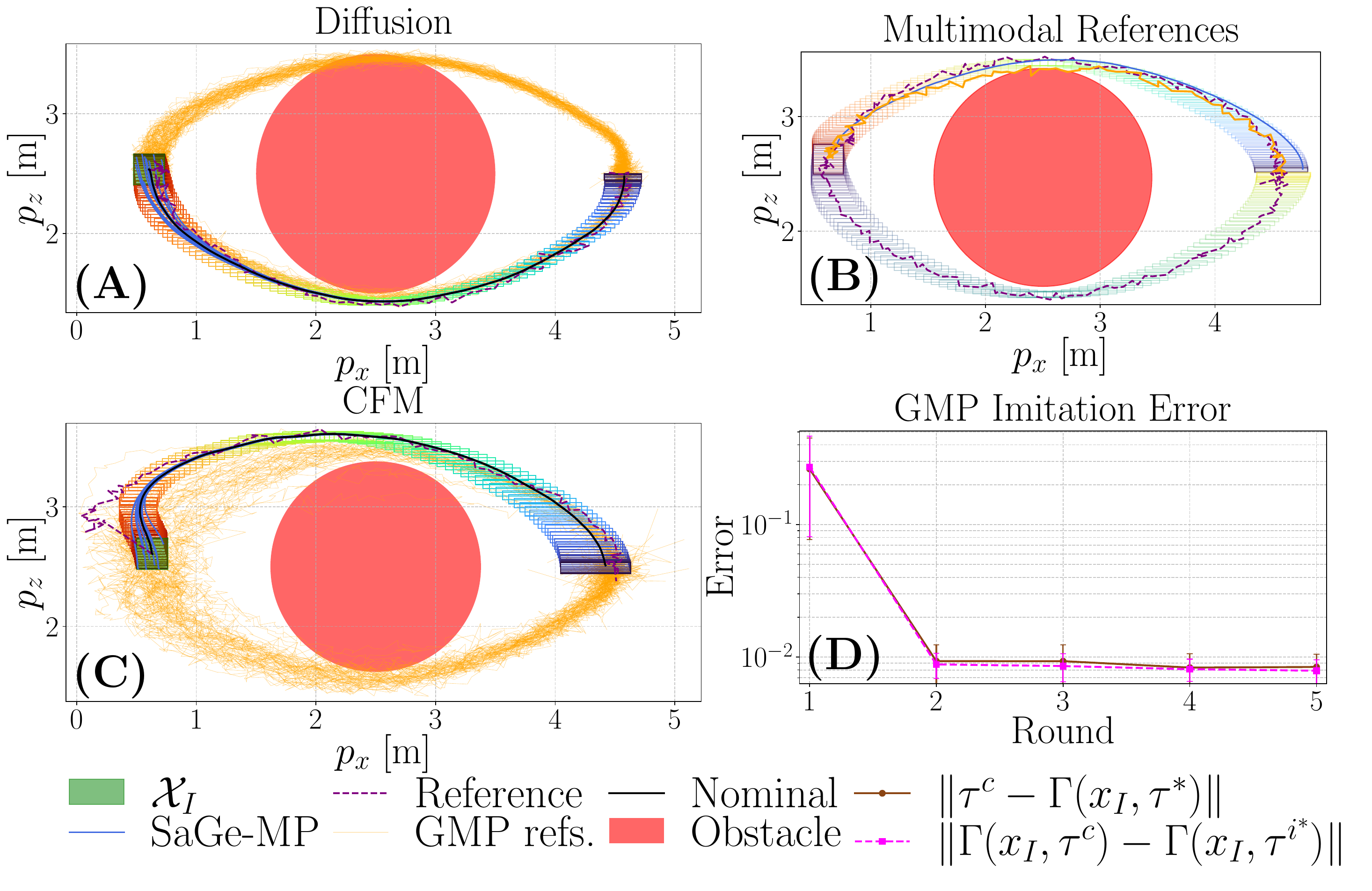}
    \vspace{-20pt}
    \caption{We find one safe RSOA via Alg. \ref{alg:safe_sampling} for diffusion (A) and CFM (C) planners. We illustrate our trajectory library approach (B), which can capture multiple modes of the GMP and reduce the imitation error relative to the original GMP (D).}
    \label{fig:patching}
\end{figure}

\looseness-1We validate that generating multiple plans from $\mathcal{X}_I$ preserves the multimodality of the GMP. For analytical planar quadcopter dynamics (Fig. \ref{fig:patching}c), at a deliberately selected $\mathcal{X}_I$, the diffusion-based planner produces two modes - above and below the obstacle. We sample five safe references from the GMP and compute their RSOAs, visualizing two in Fig. \ref{fig:patching}c. At runtime, an initial state $x_I \in \bigcup_{i=1}^C \B_\epsilon(x_I^i)$ is measured, and the GMP is queried to return a reference trajectory that starts from $x_I$. We follow Alg. \ref{alg:safe_runtime_execution}, showing that tracking the selected reference with $\pi_\theta$ can well-imitate the GMP trajectory distribution even with $C = 2$. As $C$ increases, the imitation error \eqref{eq:closed_loop_imitation_error} decreases (Fig. \ref{fig:patching}d); the raw error between the original GMP reference $\tau$ and our closed-loop rollout $\tnfl(x_I, \tau^{i^*})$ also decreases. Similarly, with analytical unicycle dynamics and a VLM-based GMP, our method certifies safe RSOAs around several sampled references; three are shown in Fig. \ref{fig:VLM_multimodal}a. Moreover, as a surrogate for measuring the increase in safety rate when starting from the certified ROA $\mathcal{A}$ relative to the GMP, we are able to certify 100\% safety, i.e., $\mathcal{X}_I \subseteq \mathcal{A}$, while na\"ively sampling and stabilizing plans with $\pi_\theta$ from the VLM leads to only 16\% safety over 50 trials, shown in Fig. \ref{fig:VLM_multimodal}b. These results suggest our method can capture multimodal GMP behavior while improving safety rates relative to the original GMP. Moreover, this shows that the invasiveness of our approach decreases as $C$ increases, with small error even for small $C$.

\vspace{-4pt}
\begin{figure}[h]
    \includegraphics[width=\linewidth]{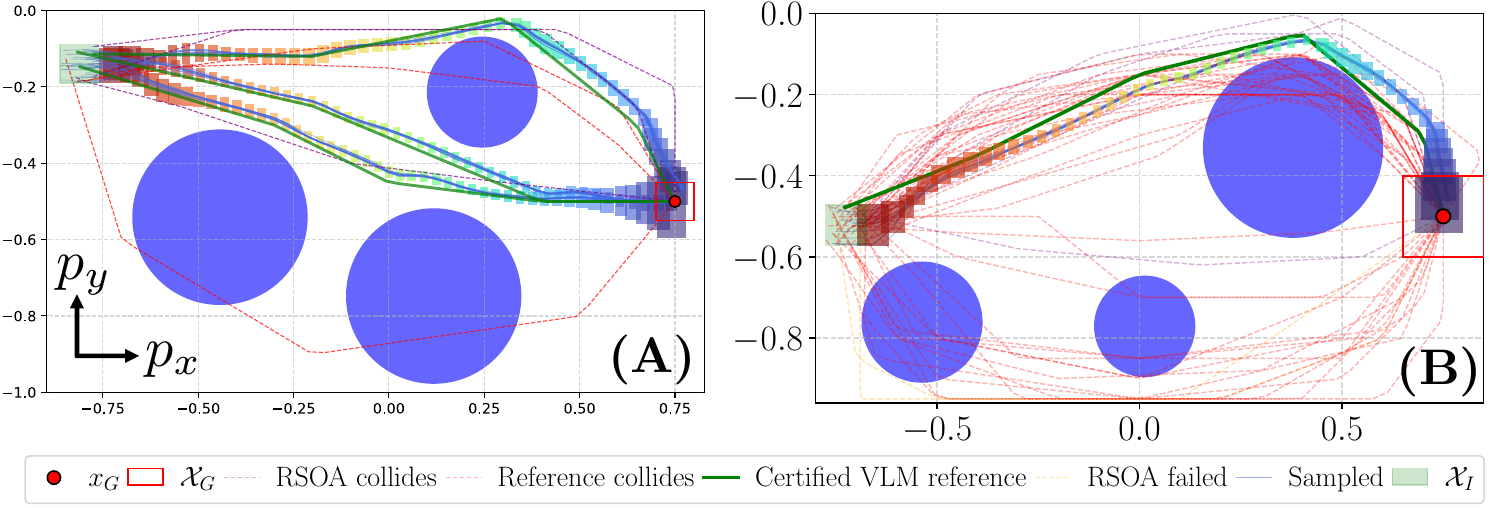}
    \vspace{-1em}
    \caption{(A) We illustrate our trajectory-library sampling approach on a VLM planner. VLM references that are unsafe (red) or have RSOAs in collision (purple) are discarded until a set of certified-safe RSOAs and references are collected (green). (B) Na\"ively executing the VLM as a planner leads to only 16\% safety starting from the green $\mathcal{X}_I$, whereas with our method enables 100\% safety.}
    \label{fig:VLM_multimodal}\vspace{-4pt}
\end{figure}

\vspace{-6pt}
\subsection{Verifying forward invariance for large $\mathcal{X}_I$ sets}\label{sec:results_large}
\looseness-1We demonstrate blending GMP samples with the tracking controller $\pi_\theta$ to verify safe goal reachability for large initial sets $\mathcal{X}_I$. We partition $\mathcal{X}_I$ into a grid, sample a NODE-based reference from each cell center, and compute the corresponding RSOA. Fig. \ref{fig:2dquad_funnel_lib} shows certified collision-free RSOAs (green) and those in collision (red). Once computed offline, these RSOAs allow fast evaluation of which subsets of $\mathcal{X}_I$ remain safe under new obstacle configurations, via collision-checking between hyper-rectangular RSOAs and the environment, taking $0.122\pm 0.003$s seconds. While grid-based partitioning works best in low dimensions, this experiment shows pre-certification enables efficient real-time safety queries; moreover, to reduce partitioning in higher dimensions, one can employ adaptive gridding.

\vspace{-8pt}
\begin{figure}[htbp]
    \centering
    \begin{minipage}[c]{0.48\linewidth}
        \includegraphics[width=0.98\linewidth]{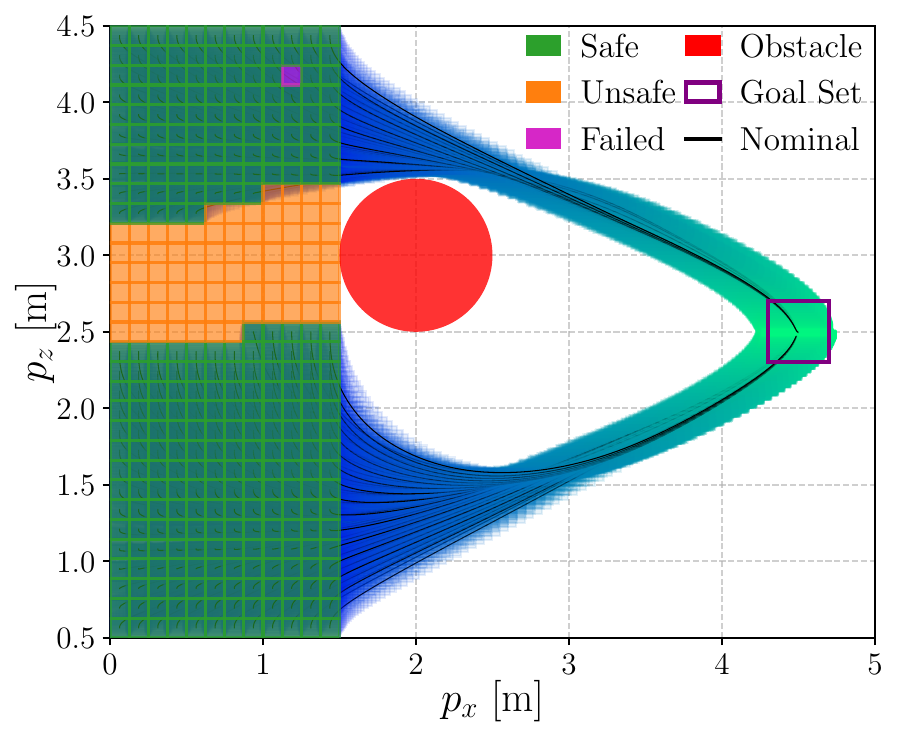}
    \end{minipage}
    \begin{minipage}[c]{0.5\linewidth}
        \caption{We verify forward invariance for a large set of initial conditions (green) by blending our tracking controller with several sampled references from a NODE-based planner, on analytical planar quadcopter dynamics.}
        \label{fig:2dquad_funnel_lib}
    \end{minipage}
\end{figure}
 
\vspace{-14pt}
\section{Conclusion}

\looseness-1We present SaGe-MP, a method for formally verifying the safety and dynamic feasibility of motion plans from large GMPs. Our approach stabilizes GMP-sampled references with a small neural tracking controller and applies NNV to the resulting closed-loop dynamics, enabling rigorous reachability guarantees even when the GMP is too large for direct analysis. By certifying multiple references offline and deploying them online via a trajectory library, our method preserves the GMP’s multimodal behavior while increasing safety rate when starting from a certified-safe region of attraction. 
In future work, we will explore tighter NN relaxations, e.g., those based on branch-and-bound, to reduce the conservativeness of our approach and enable faster reachable set calculations to enable safe deployment in nonstationary environments.

\vspace{-6pt}
\bibliographystyle{IEEEtran}

\appendix

\subsection{Experiment Details}\label{app:parameters}

\noindent\textbf{Dynamics:} For planar quadrotor examples, we used the parameters $m=0.5kg$, $g=-9.81m/s^2$, $I_y=0.01kg \cdot m^2$ for the dynamics model
\begin{equation}
    \dot{\mathbf{x}}\;=\; f(\mathbf{x},\mathbf{u})=\begin{bmatrix}
        x_4 \\ x_5 \\ x_6 \\ -\frac{u_1}{m} \sin(x_3)\\ g+\frac{u_1}{m}\cos(x_3)\\\frac{u_2}{I_y}
    \end{bmatrix}
\end{equation}

For 3D quadrotor examples, we used the parameters $m=1$kg, $g=-9.81$ m/s$^2$, $I_{x,y,z}=[0.5,0.1,0.3]$ kg·m$^2$ for the dynamics in 

\begin{equation*}
    \dot{\mathbf{x}} \;=\; f(\mathbf{x},\mathbf{u})=
    \begin{bmatrix}
        \dot{x} \\
        \dot{y} \\
        \dot{x} \\
        q \cdot \sin(\phi)/\cos{\theta}+r\cdot \cos{\phi}/\cos{\theta} \\
        q \cdot \cos{\phi} - r\cdot \sin{\phi} \\
        p + q \cdot \sin{\phi}\cdot \tan{\theta} +r\cdot \cos{\phi} \cdot \tan{\theta} \\
        \frac{u_1}{m} \cdot (\sin{\phi} \cdot \sin{\psi} + \cos{\phi} \cdot \cos{\psi} \cdot \sin{\theta}) \\ 
        \frac{u_1}{m} \cdot (\cos{\phi} \cdot \sin{\phi} - \cos{\phi} \cdot \sin{\psi} \cdot \sin{\theta}) \\ 
        g + u_1 \cdot (\cos{\phi} \cdot \cos{\theta}) / m\\
        ((I_y - I_z) / I_x) \cdot q\cdot r + \frac{u_2}{I_x} \\
        ((I_z - I_x) / I_y) \cdot p\cdot r + \frac{u_3}{I_y} \\
        ((I_x - I_y) / I_z) \cdot p\cdot q + \frac{u_4}{I_z} \\
    \end{bmatrix}
\end{equation*}

\begin{table}[ht]
  \centering
  \caption{Initial set $\mathcal{X}_I$ and disturbance $w_k$ used across experiments for unicycle and planar quadrotor dynamics.}
  \label{tab:initset_disturbance_compact}
  \resizebox{\columnwidth}{!}{%
  \begin{tabular}{@{}lcc@{}}
    \toprule
    \textbf{Case} & \boldmath$\mathcal{X}_I$\unboldmath & \boldmath$w_k$\unboldmath \\
    \midrule
    Unicycle + GT              & $x_I\!\pm[1.5,1.5,0,0]$                    & $\pm[0.05,0.05,0,0.025]$ \\
    Unicycle + Learned               & $x_I\!\pm[1.5,1.5,0,0]$                       & $\pm[0.05,0.05,0,0.025]$ \\
    Unicycle + VLM               & $x_I\!\pm[0.05,0.05,0]$                     & $\pm[0.01,0.01,0.05]$   \\
    2D Quad + GT     & $x_I\!\pm[0.15,0.15,0,0,0,0]$               & $\pm[0,0,0.25,0.25,0.25]$ \\
    2D Quad + Learned            & $x_I\!\pm[0.05,0.05,0,0,0,0]$               & $\pm[0,0,0,0.01,0.01,0.01]$ \\
    \bottomrule
  \end{tabular}%
  }
\end{table}

\begin{table}[ht]
  \centering
  \caption{Initial set $\mathcal{X}_I$ and disturbance $w_k$ used in the experiment for 3D quadrotor dynamics.}
  \label{tab:quad3d_gt_initdist}
  \resizebox{\columnwidth}{!}{%
  \begin{tabular}{@{}ll@{}}
    \toprule
    \textbf{3D Quad + GT} & \\ \midrule
    \boldmath$\mathcal{X}_I$\unboldmath & $x_I \pm [0.1,\,0.1,\,0.1,\,0,\,0,\,0,\,0,\,0,\,0,\,0,\,0,\,0]$ \\
    \boldmath$w_k$\unboldmath           & $\pm[0,\,0,\,0,\,0,\,0,\,0,\,0.1,\,0.1,\,0.1,\,0.1,\,0.1,\,0.1]$ \\
    \bottomrule
  \end{tabular}%
  }
\end{table}

\newpage
\noindent\textbf{Prompt for Fig. \ref{fig:real_robot_vlm}a:}

\begin{tcolorbox}[colback=gray!10, 
                  colframe=gray!50, 
                  arc=4mm,          
                  boxrule=0.5pt,    
                  width=\linewidth, 
                  ]

You are an expert motion planner for a mobile robot. Your task is to analyze the provided image and generate a safe and efficient series of intermediate waypoints for the robot to navigate from a start point to a goal point.

\noindent Environment Details:

\begin{itemize}
    \item Coordinate System: The image represents a 2D environment with:
    \begin{itemize}[label=$\circ$]
        \item The $x$-axis ranges from $-1.0$ (left edge) to $1.0$ (right edge).
        \item The $y$-axis ranges from $-1.0$ (bottom edge) to $1.0$ (top edge).
        \item The origin $(0.0, 0.0)$ is at the exact center of the image.
    \end{itemize}
\end{itemize}

\noindent Task Definition:

\begin{itemize}
    \item Start Point (Green Dot): [current start]
    \item Goal Point (Red Dot): [goal]
    \item Waypoints: Your goal is to find the most efficient and direct path possible.
    \begin{itemize}
        \item Actively look for and utilize safe passages or gaps between obstacles.
        \item The path should be smooth and avoid abrupt bends. The waypoints you provide will be linearly interpolated.
    \end{itemize}
\end{itemize}

\noindent Safety Constraints:

\begin{itemize}
    \item All food and drink items in the image are obstacles that the robot must avoid.
\end{itemize}

\noindent Required Output Format:

\begin{itemize}
    \item Your entire response MUST be a single JSON object and nothing else. Do NOT include any explanatory text, greetings, apologies, or markdown formatting like \texttt{```json}.
\end{itemize}

\end{tcolorbox}

\newpage
\noindent\textbf{Prompt for Fig. \ref{fig:real_robot_vlm}b:}
\begin{tcolorbox}[colback=gray!10, 
                  colframe=gray!50, 
                  arc=4mm,          
                  boxrule=0.5pt,    
                  width=\linewidth, 
                  ]
You are an expert motion planner for a mobile robot. Your task is to analyze the provided image and generate a safe and efficient series of intermediate waypoints for the robot to navigate from a start point to a goal point.
**Environment Details:**
* **Coordinate System:** The image represents a 2D environment with:
    * The x-axis ranges from -1.0 (left edge) to 1.0 (right edge).
    * The y-axis ranges from -1.0 (bottom edge) to 1.0 (top edge).
    * The origin (0.0, 0.0) is at the exact center of the image.
**Task Definition:**
* **Start Point (Green Dot):** $[{\textrm{start\_np}[0]:.2f}, {\textrm{start\_np}[1]:.2f}]$
* **Goal Point (Red Dot):** $[{\textrm{goal\_np}[0]:.2f}, {\textrm{goal\_np}[1]:.2f}]$
* **Waypoints:** Your goal is to find the most **efficient and direct path** possible.
    * **Actively look for and utilize safe passages or gaps *between* obstacles.**
    * The path should be smooth and avoid abrupt bends. The waypoints you provide will be linearly interpolated.
**Safety Constraints:**
* All food and drink items, **except the burger**, in the image are **obstacles** that the robot must avoid.
**Special Instructions:**
* The author is craving some food. Find the burger and please ensure the path goes over the burger.
**Required Output Format:**
Your entire response MUST be a single JSON object and nothing else. Do NOT include any explanatory text, greetings, apologies, or markdown formatting like ```json.
\end{tcolorbox}

\noindent\textbf{Prompt for Fig. \ref{fig:VLM_multimodal}a-b:}
\begin{tcolorbox}[colback=gray!10, 
                  colframe=gray!50, 
                  arc=4mm,          
                  boxrule=0.5pt,    
                  width=\linewidth, 
                  ]

You are an expert and cautious motion planner for a mobile robot operating in a 2D environment. Your primary goal is to find a safe path from a start point to a goal point, completely avoiding all obstacles.
\begin{itemize}
    \item Coordinate System: x from -1.0 to 1.0, y from -1.0 to 0.0.
    \item Start Point (Green Dot): [current start]
    \item Goal Point (Red Dot): [goal]
    \item Obstacles: Blue circles (no-go zones).
\end{itemize}
Task: Provide a series of [x, y] waypoints for a smooth, safe path around obstacles.

\noindent Critical Requirement: DO NOT generate a path that intersects any blue obstacles.

\noindent Output Format: Your entire response MUST be a single, raw JSON object. Example: {{"waypoints": [[-0.5, -0.2], [0.1, -0.5]]}}

\end{tcolorbox}

\end{document}